\documentclass{sig-alternate}
\usepackage{amsmath,amssymb, graphicx}
\usepackage[ruled,linesnumbered]{algorithm2e}
\usepackage{algorithmic}
\usepackage{mdwlist}
\usepackage{color}
\usepackage{multirow}
\usepackage{rotating}
\usepackage{slashbox}
\usepackage{times}
\usepackage[mathscr]{eucal} 
\usepackage{amsbsy} 
\usepackage{subfigure}
\usepackage{booktabs}
\usepackage{tikz}
\usetikzlibrary{snakes}
\usepackage{multirow}
\usepackage{epstopdf}

\usepackage{enumitem}
\setlist{nolistsep}

\usepackage[labelfont=bf]{caption}
\DeclareCaptionType{copyrightbox}
\usepackage{url}

\newcounter{ALC@tempcntr}
\newcommand{\LCOMMENT}[1]{
    \setcounter{ALC@tempcntr}{\arabic{ALC@rem}}
    \setcounter{ALC@rem}{1}
    \item \{#1\}
    \setcounter{ALC@rem}{\arabic{ALC@tempcntr}}
}%

\newtheorem{lemma}{Lemma}

\newcommand{\hide}[1]{}

\newcommand{\tensor}[1]{\underline{\mathbf{#1}}}

\def\blambda{\mbox{\boldmath ${\lambda}$}}

\newcommand{\method}{\textsc{Scoup-Smt}\xspace}
\newcommand{\methodplain}{Scoup-SMT\xspace}

\newcommand{\merge}{\textsc{Merge}\xspace}

\newcommand{\facebook}{\textsc{Facebook}\xspace}

\newcommand{\brain}{\textsc{BrainQ}\xspace}

\newcommand{\ben}{\begin{enumerate*}}
\newcommand{\een}{\end{enumerate*}}
\newcommand{\bit}{\begin{itemize*}}
\newcommand{\eit}{\end{itemize*}}

\begin{document}


\title{\methodplain: Scalable Coupled Sparse Matrix-Tensor Factorization}

\numberofauthors{4}

\author{
\alignauthor
Evangelos E. Papalexakis\\
       \affaddr{Carnegie Mellon University}\\
       \email{epapalex@cs.cmu.edu}
\and
\alignauthor
Tom M. Mitchell\\
       \affaddr{Carnegie Mellon University}\\
       \email{tom.mitchell@cmu.edu}
\and
\alignauthor
Nicholas D. Sidiropoulos\\
       \affaddr{University of Minnesota}\\
       \email{nikos@ece.umn.edu}
\and
\alignauthor
Christos Faloutsos\\
       \affaddr{Carnegie Mellon University}\\
       \email{christos@cs.cmu.edu}
\and
\alignauthor
Partha Pratim Talukdar \\
	\affaddr{Carnegie Mellon University}\\
	\email{partha.talukdar@cs.cmu.edu}
	\and
\alignauthor
Brian Murphy \\
	\affaddr{Carnegie Mellon University}\\
	\email{brianmurphy@cmu.edu}
}

\toappear{}

\maketitle

\setlength{\floatsep}{0.1cm}
\setlength{\textfloatsep}{0.1cm}
\setlength{\intextsep}{0.1cm}
\setlength{\dblfloatsep}{0.1cm}
\setlength{\abovedisplayskip}{0.1cm}
\setlength{\belowdisplayskip}{0.1cm}

\begin{abstract}
How can we correlate neural activity in the human brain 
as it responds to words, 
with behavioral data expressed as answers to questions about 
these same words?   
In short, we want to find latent variables, that explain
both the brain activity, as well as the behavioral responses.
We show that this is 
an instance of the \emph{Coupled Matrix-Tensor Factorization} (CMTF) problem.
We propose \method, a novel, fast, and parallel algorithm 
that solves the CMTF problem 
and produces a \emph{sparse} latent low-rank subspace of the data. 
In our experiments, we find that \method is {\em 50-100 times }
faster than a state-of-the-art algorithm for CMTF, 
along with a {\em 5 fold} increase in sparsity. 
Moreover, we extend \method to handle missing data 
without degradation of performance.

We apply \method to \brain, 
a dataset consisting of a (nouns, brain voxels, human subjects) tensor 
and a (nouns, properties) matrix, 
with coupling along the nouns dimension. 
\method is able to find meaningful latent variables,
as well as to predict brain activity with competitive accuracy.
Finally, we demonstrate the generality of \method, 
by applying it on a \facebook dataset (users, 'friends', wall-postings);
there, \method spots spammer-like anomalies.

\end{abstract}



\vspace{-0.3cm}

\keywords{Tensor Decompositions, Coupled Matrix-Tensor Factorization, Sparsity, Parallel Algorithm, Brain Activity Analysis }

\section{Introduction}
\label{sec:intro}
How is knowledge mapped and stored in the human brain? How is it expressed by people answering simple questions about specific words? If we have data from both worlds, are we able to combine them and jointly analyze them? In a very different scenario, suppose we have the social network graph of an online social network, and we also have additional information about how and when users interacted with each other. What is a comprehensive way to combine those two pieces of data? Both, seemingly different, problems may be viewed as instances of what is called  \emph{Coupled Matrix-Tensor Factorization} (CMTF), where a data tensor and matrices that hold additional information are jointly decomposed into a set of low-rank factors.

In this work, we introduce \method, a fast, scalable, and sparsity promoting CMTF algorithm.
Our main contributions are the following:
\begin{itemize}[noitemsep]
	\item \emph{Fast, parallel  \& sparsity promoting algorithm:} We provide a novel, scalable, and sparsity inducing algorithm, \method,  that jointly decomposes coupled matrix-tensor data. Figure \ref{fig:crown} shows the accuracy of \method (compared to the traditional algorithm), as a function of portion of the wall-clock time that our algorithm took, again compared to the traditional one. The result indicates a speedup of about 50-100 times, while maintaining very good accuracy.\footnote{Accuracy or relative cost is defined in Section \ref{sec:exp} as the ratio of the squared approximation error of \method, divided by that of the traditional ALS algorithm.}
	\item \emph{Robustness to missing data:} We carefully derive an improved version of the above algorithm which is resilient to missing data and performs well, even with a large portion of the entries missing.
	\item \emph{Effectiveness \& Knowledge Discovery:} We analayze \brain, a brain scan dataset which is coupled to a semantic matrix (see Sec. \ref{sec:discovery} for details).
The brain scan part of the dataset consists of fMRI scans first used in \cite{mitchell2008predicting}, a work that first demonstrated that brain activity can be predictably analyzed into component semantic features. Here, we demonstrate a disciplined way to combine both datasets and carry out a variety of data mining/machine learning tasks, through this joint analysis.


\item \emph{Generality}:  We illustrate the generality of our approach, by applying \method to a completely different setting of a time-evolving social network with side information on user interactions, demonstrating \method's ability to discover anomalies.

\end{itemize}

\begin{figure}[!htf]
	\begin{center}
		\includegraphics[width = 0.45\textwidth]{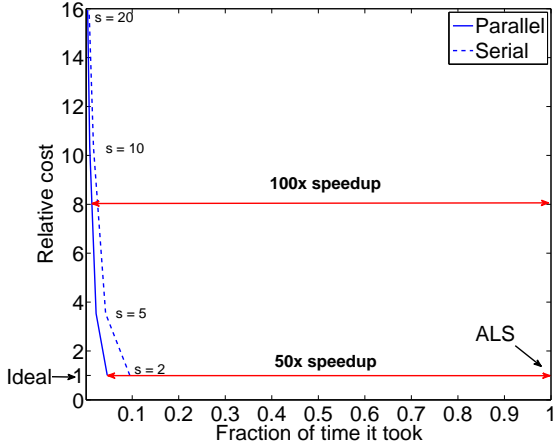}
	\end{center}
	\caption{The relative cost of \method (with respect to the ALS algorithm) as a function of the fraction of the wall-clock time of ALS that the computation required, vividly demonstrates the gains of \method in terms of speedup. In particular, for the entire \brain dataset which is very dense (see Sec. \ref{sec:discovery}), the speedup incurred by the parallel version of \method on 4 cores, was in the range of 50-100 times. This Figure also shows the behavior of \method with respect to the sampling parameter $s$. As $s$ increases, \method runs faster but the relative cost increases as well.}
	\label{fig:crown}
\end{figure}

\hide{
  \begin{minipage}[b]{0.5\textwidth}
  \begin{minipage}[b]{0.5\textwidth}
    	\includegraphics[width = \textwidth]{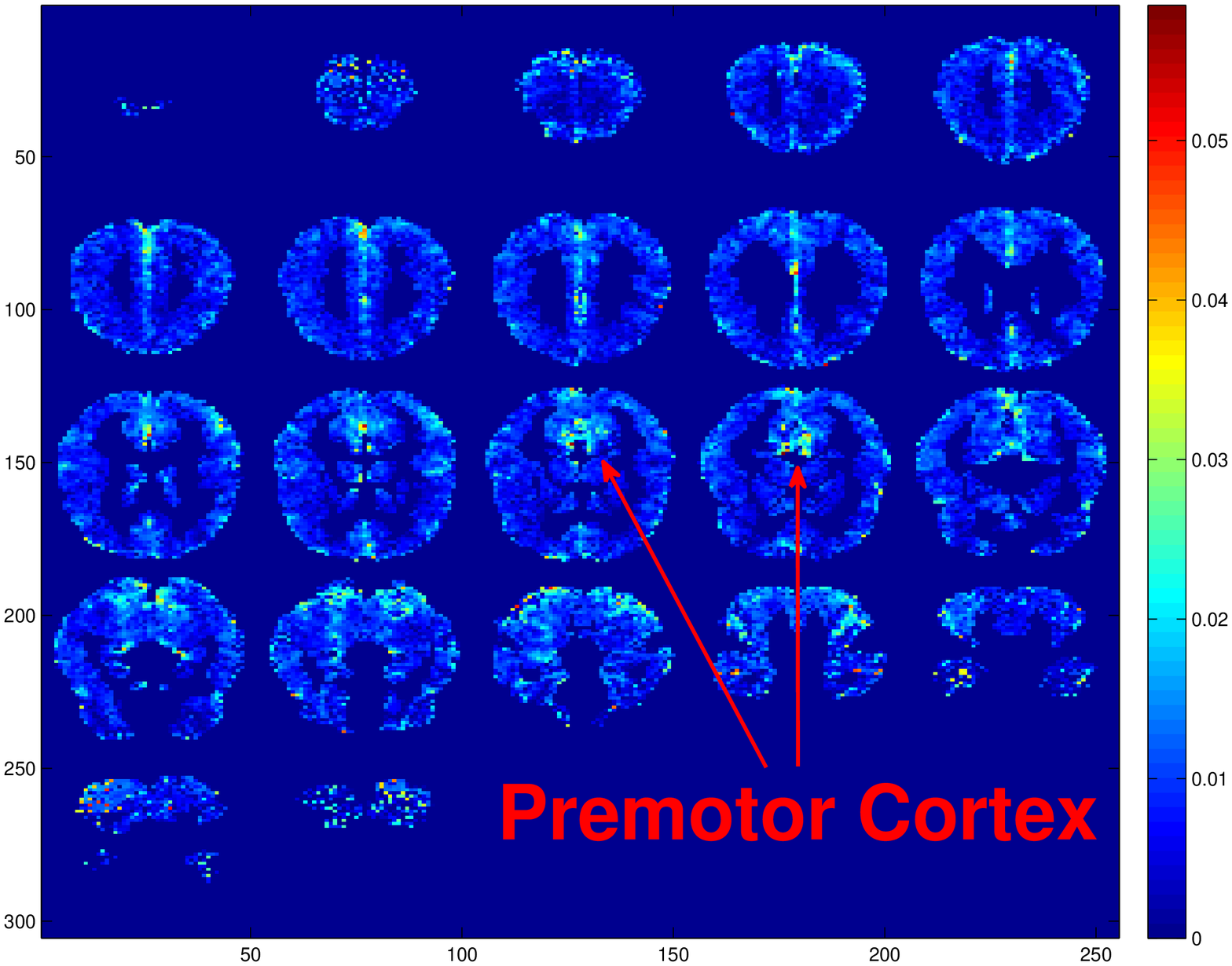}
  \end{minipage}
  \begin{minipage}[b]{0.45\textwidth}
    \hspace{-0.1in}
    \centering
    \scriptsize
  \begin{tabular}{lll}
  \vspace{-1.4in}
  \\
     &\textbf{Nouns}  \\ 
	\multirow{4}{*}{ } &glass, tomato	, bell	\\ 
	     & \textbf{Questions} \\ 
	\multirow{4}{*}{ } 	& can you pick it up?	\\
		& 	can you hold it in one hand?\\
		&	is it smaller than a golfball?'\\ 
  \end{tabular}
    \end{minipage}
    \captionof{figure}{blah}
    \label{fig:crown_2}
  \end{minipage}
 }

\section{Preliminaries}
\label{sec:prelim}
 \begin{table}[!h]
  \begin{center}
{
  \begin{tabular}{lll}
    & \textbf{Symbol} & \textbf{Description} \\
    & CMTF & Coupled Matrix-Tensor Factorization \\
    & ALS & Alternating Least Squares \\
    & $x, \mathbf{x}, \mathbf{X}, \tensor{X}$ & scalar, vector, matrix, tensor (respectively)  \\ 
    &$\mathbf{A\odot B}$& Khatri-rao product (see \cite{kolda2009tensor}). \\
    & $\mathbf{A\ast B}$ & Hadamard (elementwise) product. \\
    &$\mathbf{A}^\dag$ 	&	Pseudoinverse of $\mathbf{A}$ (see Sec. \ref{sec:prelim})\\
    & $\| \mathbf{A} \|_F$ & Frobenius norm of $\mathbf{A}$.  \\
    &$\mathbf{a\circ b\circ c}$& $\left(\mathbf{a\circ b\circ c}\right)(i,j,k) = \mathbf{a}(i)\mathbf{b}(j)\mathbf{c}(k)$\\
    &$(i)$ as superscript & Indicates the $i$-th iteration\\
    &$\mathbf{A}_1^{i}$, $\mathbf{a}_1^{i}$& series of matrices or vectors, indexed by $i$.\\
    & $\tensor{X}_{(i)}$& $i$-th mode unfolding of tensor $\tensor{X}$ (see \cite{kiers2000towards}).\\
    &$\mathcal{I}$&	Set of indices.\\
    &$\mathbf{x}(\mathcal{I})$ & Spanning indices $\mathcal{I}$ of $\mathbf{x}$.	\\

  \end{tabular}
  }
  \end{center}
  \caption{Table of symbols}
  \label{tab:sybmols}
\end{table}

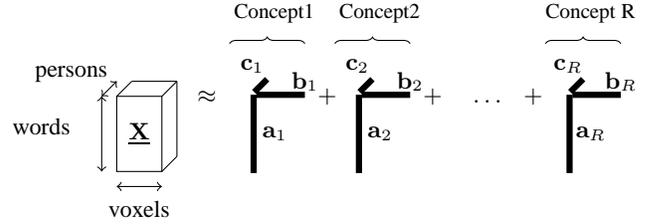
\begin{figure}[htbp]
\begin{center}

	\begin{tikzpicture}[scale=0.2]
	
	       \foreach \x in { (0,0) }
	       {
	           \draw \x rectangle +(3,-5);
		   \draw \x -- +(1,1);
		   \draw \x+(3,0) -- +(3+1,1);
		   \draw \x+(3,-5) -- +(3+1,-5+1);
		   \draw \x+(1,1) -- +(3+1,1+0);
		   \draw \x+(3+1,1) -- +(3+1,-5+1+0);
	       };
	       \node at (1.5,-2.5) {$\tensor{X}$};
	       \node at (6,0) {$\approx$};
	       \draw [<->] (0,-6) -- +(3,0);
	       \node (users) at (1.5,-7.5) {voxels};
	
	       \draw [<->] (-1,0) -- +(0,-5);
	       \node (urls) at (-5,-2) {words};
	
	       \draw [<->] (-1,0) -- +(1,1);
	       \node (date) at (-3,1.5) {persons};

	       \foreach \ind/\sub/\step/\offset in {1/1/7/2, 2/2/7/2}{
		   \node (root) at (\ind*\step+\offset,0) {};
	           \draw[black,very thick] (root) rectangle +(0.2,-5);
	           \draw[black,very thick](root) +(0.4,0) rectangle +(0.4+3,0.2);
	           \draw[black,very thick, rotate around={45:(root)+(0.4,0)}]
		       (root)+(0.4,0) rectangle +(0.4+1,0.2);
	           \node at (\ind*\step+\offset+1.5,-2.5)     {$\mathbf{a}_{\sub}$};
	           \node at (\ind*\step+\offset+3.5,1)     {$\mathbf{b}_{\sub}$};
	           \node at (\ind*\step+\offset+0,2)     {$\mathbf{c}_{\sub}$};
	           \node (plus \ind) at (\ind*\step+\offset+ 5, 0)     {$+$};
	           \draw [snake=brace] (\ind *\step + \offset -1.5,3.5) -- +(5,0);
	           \node at (\ind * \step + \offset + 1.5, 5.5) {\small Concept\sub };
	       };
	
	       \foreach \ind/\sub/\step/\offset in {4/R/7/2}{
		   \node (root) at (\ind*\step+\offset,0) {};
	           \draw[black,very thick] (root) rectangle +(0.2,-5);
	           \draw[black,very thick](root) +(0.4,0) rectangle +(0.4+3,0.2);
	           \draw[black,very thick, rotate around={45:(root)+(0.4,0)}]
		       (root)+(0.4,0) rectangle +(0.4+1,0.2);
	           \node at (\ind*\step+\offset+1.5,-2.5)     {$\mathbf{a}_{\sub}$};
	           \node at (\ind*\step+\offset+3.5,1)     {$\mathbf{b}_{\sub}$};
	           \node at (\ind*\step+\offset+0,2)     {$\mathbf{c}_{\sub}$};
	           \draw [snake=brace] (\ind *\step + \offset -1.5,3.5) -- +(5,0);
	           \node at (\ind * \step + \offset + 1.5, 5.5) {\small Concept \sub };
	       };
	
	       \node [right of=plus 2] {$\ldots ~ ~~ +$};

	\end{tikzpicture} 
    \vspace{-1em}
    \caption{\small PARAFAC decomposition of a three-way tensor of a brain activity tensor
    as sum of $F$ outer products (rank-one tensors),
    reminiscing of the rank-$F$ singular value decomposition of a matrix. Each component corresponds to a {\bf latent} concept of, e.g. "insects", "tools" and so on, a set of brain regions that are most active for that particular set of words, as well as groups of persons. 
    \label{fig:decomp} }
\end{center}
\end{figure}

\subsection{Introduction to Tensors}
Matrices record dyadic properties, like ``people recommending products''. Tensors are
the $n$-mode generalizations, capturing 3- and higher-way relationships. For example
 ``subject-verb-object'' relationships, such as the ones recorded by the Read the Web - NELL project \cite{RTW} (and have been recently used in this context \cite{kang2012gigatensor} \cite{papalexakis2012parcube}) naturally lead to a 3-mode tensor. In this work, our working example of a tensor has three modes. The first mode contains a number of nouns; the second mode corresponds to the brain activity, as recorded by an fMRI machine; and the third mode identifies the human subject corresponding to a particular brain activity measurement.

Earlier \cite{papalexakis2012parcube}
we introduced a scalable and parallelizable tensor decomposition 
which uses mode sampling. In this work, we focus on a more general and expressive framework, that of \emph{Coupled Matrix-Tensor Factorizations}.


\subsection{Coupled Matrix-Tensor Factorization}
Oftentimes, two tensors, or a matrix and a tensor, may have one mode in common;
consider the example that we mentioned earlier, where we have a word by brain activity by human subject tensor, we also have a semantic matrix that provides additional information for the same set of words.
In this case,
we say that the matrix and the tensor are {\em coupled} in the 'subjects' mode.

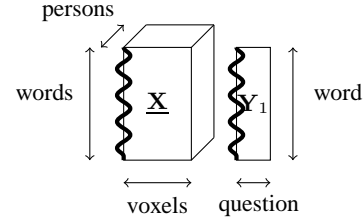
\begin{figure}[!htp]

	\begin{center}
	    \begin{tikzpicture}[scale=0.3]

       \foreach \x in { (0,0) }
       {
           \draw \x rectangle +(3,-5);
           \draw [snake=snake, ultra thick] \x -- +(0,-5); 
           \draw \x -- +(1,1);
           \draw \x+(3,0) -- +(3+1,1);
           \draw \x+(3,-5) -- +(3+1,-5+1);
           \draw \x+(1,1) -- +(3+1,1+0);
           \draw \x+(3+1,1) -- +(3+1,-5+1+0);
       };

       \node at (1.5,-2.5) {$\tensor{X}$};
       \draw [<->] (0,-6) -- +(3,0);
       \node (users) at (1.5,-7) {voxels};

       \draw [<->] (-1.5,0) -- +(0,-5);
       \node (urls) at (-3.5,-2) {words};

       \draw [<->] (-1,0) -- +(1,1);
       \node (date) at (-2,1.5) {persons};

       \draw (5,0) rectangle (6.5,-5);
       \draw [snake=snake, ultra thick] (5,0) -- +(0,-5); 
       \node (Ytensor) at (5.75,-2.5) {$\mathbf{Y}_1$};

       \draw [<->] (5,-6) -- +(1.5,0);
       \node (attrs) at (6,-7) {question};

       \draw [<->] (7.5,0) -- +(0,-5);
       \node (urls2) at (9.5,-2) {word};

\end{tikzpicture}
	
	\caption{\small {\em Coupled Matrix - Tensor} example:
            Tensors often share one or more modes
            (with thick, wavy line):
			$\tensor{X}$ is the brain activity tensor and $\mathbf{Y}$ is the semantic matrix. As the wavy line indicates, these two
			datasets are coupled in the 'word' dimension.
	\label{fig:multiset} }
	\end{center}
\vspace{-0.1in}
\end{figure}

In this work we focus on three mode tensors, however, everything we mention extends directly  to higher modes. In the general case, a three mode tensor $\tensor{X}$ may be coupled with at most three matrices $\mathbf{Y}_i,~i=1\cdots3$, in the manner illustrated in Figure \ref{fig:multiset} for one mode. The optimization function that encodes this decomposition is:
\begin{align}
\label{eq:main_obj}
	&\displaystyle{ \min_{\mathbf{A,B,C,D,E,G}} \| \tensor{X} - \sum_k \mathbf{a}_k \circ \mathbf{b}_k\circ \mathbf{c}_k\|_F^2  ~+  }  \\
	&\displaystyle{\|  \mathbf{Y}_1 - \mathbf{AD}^T\|_F^2     +  \| \mathbf{Y}_2 - \mathbf{BE}^T\|_F^2  + \| \mathbf{Y}_3 - \mathbf{CG}^T\|_F^2}  \nonumber
\end{align}
where $\mathbf{a}_k$ is the $k$-th column of $\mathbf{A}$. The idea behind the coupled matrix-tensor decomposition is that we seek to jointly analyze $\tensor{X}$ and $\mathbf{Y}_i$, decomposing them to latent factors who are coupled in the shared dimension. For instance, the first mode of $\tensor{X}$ shares the same low rank column subspace as $\mathbf{Y}_1$; this is expressed through the latent factor matrix $\mathbf{A}$ which jointly provides a basis for that subspace.

\subsection{The Alternating Least Squares Algorithm}
One of the most popular algorithms to solve PARAFAC (as introduced in Figure \ref{fig:decomp}) is the so-called Alternating Least Squares (ALS); the basic idea is that by fixing two of the three factor matrices, we have a least squares problem for the third, and we thus do so iteratively, alternating between the matrices we fix and the one we optimize for, until the algorithm converges, usually when the relative change in the objective function between two iterations is very small.

Solving CMTF using ALS follows the same strategy, only now, we have up to three additional matrices in our objective. For instance, when fixing all matrices but $\mathbf{A}$, the update for $\mathbf{A}$ requires to solve the following least squares problem:
\[
	\min_{\mathbf{A}} \|  \tensor{X}_{(1)} - \mathbf{(B\odot C)A}^T\|_F^2 + \|\mathbf{Y}_1 - \mathbf{DA}^T\|_F^2
\]

In Algorithm \ref{alg:als}, we provide a detailed outline of the ALS algorithm for CMTF.

In order to obtain initial estimates for matrices $\mathbf{A,B,C}$ we take the PARAFAC decomposition of $\tensor{X}$.
As for matrix $\mathbf{D}$ (and similarly for the rest), it suffices to solve a simple Least Squares problem, given the PARAFAC estimate of $\mathbf{A}$,
we initialize as $\mathbf{D} = \mathbf{Y}_1 \left(\mathbf{A}^{\dag}\right)^T$, where $\dag$ denotes the Moore-Penrose Pseudoinverse which, given the Singular Value Decomposition of a matrix
$\mathbf{X} = \mathbf{U\Sigma V}^T,$
 is computed as 
$ \mathbf{X}^\dag = \mathbf{V \Sigma^{-1} U}^T.$

 \begin{algorithm}[!ht]
\caption{Alternating Least Squares Algorithm for CMTF}
 \label{alg:als}
\begin{algorithmic}[1]
 \REQUIRE $\tensor{X}$ of size $I\times J \times K$, matrices $\mathbf{Y}_i,~i=1\cdots 3$, of size $I\times I_2$, $J\times J_2$, and $K\times K_2$ respectively, number of factors $F$.
 \ENSURE $\mathbf{A}$ of size $I \times F$, $\mathbf{b}$ of size $J\times F$, $\mathbf{c}$ of size $K\times F$,
 $ \mathbf{D} $ of size $I_2\times F$, $\mathbf{G}$ of size $J_2 \times F$, $\mathbf{E}$ of size $K_2\times F$.
 \STATE Unfold $\tensor{X}$ into $\tensor{X}_{(1)}$, $\tensor{X}_{(2)}$, $\tensor{X}_{(3)}$ (see \cite{kiers2000towards}).
 \STATE Initialize $\mathbf{A}$, $\mathbf{B}$, $\mathbf{C}$ using PARAFAC of $\tensor{X}$. Initialize $\mathbf{D,G,E}$ as discussed on the text.
 \WHILE{convergence criterion is not met}

  \STATE $ \mathbf{A} =   \begin{bmatrix}\tensor{X}_{(1)}  \\ \mathbf{Y}_1\end{bmatrix}^T \left(\begin{bmatrix} \mathbf{(B\odot C)} \\ \mathbf{D} \end{bmatrix}^\dag\right)^T  $

  \STATE  $ \mathbf{B} =  \begin{bmatrix} \tensor{X}_{(2)}  \\ \mathbf{Y}_2\end{bmatrix}^T \left(\begin{bmatrix} \mathbf{(C\odot A)} \\ \mathbf{G}  \end{bmatrix}^\dag\right)^T  $

  \STATE $ \mathbf{C} =  \begin{bmatrix} \tensor{X}_{(3)} \\\mathbf{Y}_3 \end{bmatrix}^T  \left(\begin{bmatrix} \mathbf{(A\odot B)} \\ \mathbf{E} \end{bmatrix}^\dag\right)^T$

  \STATE $\mathbf{D} = \mathbf{Y}_1 \left(\mathbf{A}^{\dag}\right)^T$,~ $\mathbf{G} = \mathbf{Y}_2 \left(\mathbf{B}^{\dag}\right)^T$,~ $\mathbf{E} = \mathbf{Y}_3 \left(\mathbf{C}^{\dag}\right)^T$

 \ENDWHILE
\end{algorithmic}
\end{algorithm}

Besides ALS, there exist other algorithms for CMTF. For example, \cite{acar2011all} uses a first order optimization algorithm for the same objective. However, we chose to operate using ALS because it is the `workhorse' algorithm for plain tensor decomposition, and it easily to incorporate additional constraints in ALS. Nevertheless, one strength of \method is that it can be used as-is with any underlying core CMTF implementation.

\section{Proposed Method}
\label{sec:method}
\hide{
\begin{itemize}[noitemsep]
	\item \emph{Fast}: \method is able to operate fast, on large datasets and provide efficient solutions very close to the exact ones. We design an algorithm which is scalable and is able to operate on big datasets, regardless of sparsity. Moreover, we carefully derive update formulas for the core ALS algorithm, which, in turn, contribute to a faster outcome.
	\item \emph{Simple}: The algorithm that outlines \method (Algorithm \ref{alg:repetition}) is concise, intuitive and relatively simple to implement.
	\item \emph{Sparse}: The factors produced by \method are \emph{sparse}, leading to easier and more intuitive interpretation of the results, as well as efficient storage thereof.
	\item \emph{Parallelizable}: \method, as outline by Algorithm \ref{alg:repetition} is readily parallelizable, as we will discuss further on.
\end{itemize}
}
\begin{figure*}[htp]
\begin{align}
\tiny
\label{eq:dens_tensor}
&\mathbf{x}_{A}(i) =\displaystyle{\sum_{j=1}^J \sum_{k=1}^K  |\tensor{X}(i,j,k)}|  + \displaystyle{ \sum_{j=1}^{I_1}| \mathbf{Y}_1(i,j) |},~
\mathbf{x}_{B}(j) =\displaystyle{\sum_{i=1}^I \sum_{k=1}^K |\tensor{X}(i,j,k)}|  + \displaystyle{ \sum_{i=1}^{I_2}| \mathbf{Y}_2(j,i) | },~
\mathbf{x}_{C}(k) =\displaystyle{\sum_{i=1}^I \sum_{j=1}^J |\tensor{X}(i,j,k)}| + \displaystyle{ \sum_{j=1}^{I_3}| \mathbf{Y}_3(k,j) | },~\\ \label{eq:dens_mat_I}
&~~~~~~~~~~~~~~~~~~~~~~~~~~~~~~~~~~~~~~~~~~~\mathbf{y}_{1,A} (i)= \displaystyle{ \sum_{j=1}^{I_1}| \mathbf{Y}_1(i,j) | }~ \mathbf{y}_{2,B}(j) =\displaystyle{ \sum_{i=1}^{I_2}| \mathbf{Y}_2(j,i) | }, ~\mathbf{y}_{3,C}(k) = \displaystyle{ \sum_{j=1}^{I_3}| \mathbf{Y}_3(k,j) | } \\~ \label{eq:dens_mat_J}
&~~~~~~~~~~~~~~~~~~~~~~~~~~~~~~~~~~~~~~~~~~~\mathbf{y}_{1,D} (j)= \displaystyle{ \sum_{i=1}^{I}| \mathbf{Y}_1(i,j) | },~ \mathbf{y}_{2,G}(i) =\displaystyle{ \sum_{j=1}^{J}| \mathbf{Y}_2(j,i) | }, ~\mathbf{y}_{3,E}(i) = \displaystyle{ \sum_{k=1}^{K}| \mathbf{Y}_3(k,i) | }
\end{align}
 \end{figure*}

\subsection{Algorithm description}
\hide{
The key ideas of the algorithm may be summarized in the following list:
\begin{enumerate}[noitemsep]
	\item Sample the data in order to get a representative, much smaller, sample of the original matrix-tensor couple. It may be possible (and preferable) to draw more than one samples of the data.
	\item Fit the CMTF model to the samples drawn from the previous step.
	\item Merge the results, while making sure that the intermediate results from the previous step are combined correctly.
\end{enumerate}
}
There are three main concepts behind \method (outlined in Algorithm \ref{alg:repetition}):
\begin{enumerate}[noitemsep]
	\item[{\bf Phase 1}] Sample the data in order to reduce the dimensionality
	\item[{\bf Phase 2}]  fit CMTF to the reduced data (possibly on more than one samples)
	\item[{\bf Phase 3}] merge the partial results
\end{enumerate}

\noindent{\bf Phase1: Sampling}
An efficient way to reduce the size of the dataset, yet operate on a representative subset thereof is to use \emph{biased} sampling.
In particular, given a three-mode tensor $\tensor{X}$ we sample as follows. We calculate three vectors as shown in equation (\ref{eq:dens_tensor}), one for each mode of $\tensor{X}$. These vectors, which we henceforth refer to as \emph{density vectors} are the marginal absolute sums with respect to all but one of the modes of the tensor, and in essence represent the importance of each index of the respective mode. We then sample \emph{indices} of each mode according to the respective density vector. For instance, assume an $I \times J \times K$ tensor; suppose that we need a sample of size $\frac{I}{s}$ of the indices of the first mode. Then, we just define
\[
p_\mathcal{I}(i) = \mathbf{x}_A (i) / \displaystyle{\sum_{i=1}^I \mathbf{x}_A (i)}
\]
as the probability of sampling the $i$-th index of the first mode, and we simply sample without replacement from the set $\{1\cdots I\}$, using $p_\mathcal{I}$ as bias. The very same idea is used for matrices $\mathbf{Y}_i$.
Doing so is preferable over sampling uniformly, since our bias makes it more probable that high density indices of the data will be retained on the sample, and hence, it will be more representative of the entire set.

Suppose that we call $\mathcal{I,J,K}$ the index samples for the three modes of $\tensor{X}$. Then, we may take $\tensor{X}_s = \tensor{X}(\mathcal{I,J,K})$ (and similarly for matrices $\mathbf{Y}_i$); essentially, what we are left with is a small, yet representative, sample of our original dataset, where the high density blocks are more likely to appear on the sample. It is important to note that the indices of the coupled modes are the same for the matrix and the tensor, e.g. $\mathbf{I}$ randomly selects the same set of indices for $\tensor{X}$ and $\mathbf{Y}_1$. This way, we make sure that the coupling is \emph{preserved} after sampling.

\noindent{\bf Phase 2: Fit CMTF to reduced data}
Having said that, the key idea of our proposed algorithm is to run ALS CMTF (Algorithm \ref{alg:als}) on the sample and then, based on the sampled indices, redistribute the result to the original index space. In more detail, suppose that $\mathbf{A}_s$ is the factor matrix obtained by the aforementioned procedure, and that jointly describes the first mode of $\tensor{X}_s$ and $\mathbf{Y}_{1,s}$. The dimensions of $\mathbf{A}_s$ are going to be $|\mathcal{I}| \times F$ (where $|\dot|$ denotes cardinality and $F$ is the number of factors). Let us further assume matrix $\mathbf{A}$ of size $I \times F$ which expresses the first mode of the tensor and the matrix, before sampling; due to sampling, it holds that $I \gg |\mathcal{I}|$. If we initially set all entries of $\mathbf{A}$ to zero and we further set
$
	\mathbf{A}(\mathcal{I},:) = \mathbf{A}_s
$
we obtain a highly \emph{sparse} factor matrix whose non-zero values are a 'best effort' approximation of the true ones, i.e. the values of the factor matrix that we would obtain by decomposing the full data.

So far, we have provided a description of the algorithm where only one repetition of sampling is used. However, if our sample consists of only a small portion of the data, inevitably, this will not be adequate in order to successfully model all variation in the data. To that end, we allow for multiple sampling repetitions in our algorithm, i.e. extracting multiple sample tensors $\tensor{X}_s$ and side matrices $\mathbf{Y}_{i,s}$, fitting a CMTF model to all of them and combining the results in a way that the true latent patterns are retained. We are going to provide a detailed outline of how to carry the multi-repetition version of \method in the following.


While doing multiple repetitions, we keep a \emph{common}. subset of indices for all different samples. In particular, let $p$ be the percentage of common values across all repetitions and $\mathcal{I}_p$ denote the common indices along the first mode (same notation applies to the rest of the indices); then, all sample tensors $\tensor{X}_s$ will definitely contain the indices $\mathcal{I}_p$ on the first mode, as well as $(1-p)\frac{I}{s}$ indices sampled independently (across repetitions) at random. This common index sample is key in order to ensure that our results are not rank deficient, and all partial results are merged correctly.

We do not provide an exact method for choosing $p$, however, as a rule of thumb, we observed that, depending on how sparse and noisy the data is, a range of $p$ between 0.2 and 0.5 works well. This introduces a trade-off between redundancy of indices that we sample, versus the accuracy of the decomposition; since we are not dealing solely with tensors, which are known to be well behaved in terms of decomposition uniqueness, it pays off to introduce some data redundancy (especially when \method runs in a parallel system) so that we avoid rank-deficiency in our data.

Let $r$ be the number of different sampling repetitions, resulting in $r$ different sets of sampled matrix-tensor couples $\tensor{X}_s^{(i)}$ and $\mathbf{Y}_{j,s}^{(i)}$ ($i=1\cdots r, ~j = 1\cdots 3$). For that set of coupled data, we fit a CMTF model, using Algorithm \ref{alg:als}, obtaining a set of factor matrices $\mathbf{A}^{(i)}$ (and likewise for the rest).

\noindent{\bf Phase 3: Merging partial results}
After having obtained these $r$ different sets of partial results, as a final step, we have to merge them together into a set of factor matrices that we would ideally get had we operated on the full dataset.

In order to make the merging work, we first introduce the following scaling on each column of each factor matrix: Let's take $\mathbf{A}^{(i)}$ for example; we normalize each column of $\mathbf{A}$ by the $\ell_2$ norm of the common part, as described in line 8 of Algorithm \ref{alg:repetition}. By doing so, the common part of each factor matrix (for all repetitions) will be unit norm. This scaling is absorbed in a set of scaling vectors $\blambda_A$ (and accordingly for the rest of the factors). The new objective function is shown in Equation \ref{eq:scaling}

\begin{align}
\label{eq:scaling}
	&\displaystyle{ \min_{\mathbf{A,B,C,D,E,G}} \| \tensor{X} - \sum_k \blambda_A(k) \blambda_B(k) \blambda_C(k)\mathbf{a}_k \circ \mathbf{b}_k\circ \mathbf{c}_k\|_F^2  ~+  }  \\ \nonumber
	&\displaystyle{\|  \mathbf{Y}_1 - \mathbf{A}~\text{diag}(\blambda_A \ast \blambda_D)~\mathbf{D}^T\|_F^2     +  \| \mathbf{Y}_2 - \mathbf{B}~\text{diag}(\blambda_B\ast \blambda_E)~\mathbf{E}^T\|_F^2  ~}\\ \nonumber
	& + \displaystyle{ \| \mathbf{Y}_3 - \mathbf{C}~\text{diag}(\blambda_C \ast \blambda_G)~\mathbf{G}^T\|_F^2}  \nonumber
\end{align}

A problem that is introduced by carrying out multiple sampling repetitions is that the correspondence of the output factors of each repetition is very likely to be distorted. In other words, say we have matrices $\mathbf{A}^{(1)}$ and $\mathbf{A}^{(2)}$ and we wish to merge their columns (i.e. the latent components) into a single matrix $\mathbf{A}$, by stitching together columns that correspond to the same component. It might very well be the case that the order in which the latent components appear in $\mathbf{A}^{(1)}$ is not the same as in $\mathbf{A}^{(2)}$.

The sole purpose of the aforementioned normalization is to resolve the correspondence problem. In Algorithm \ref{alg:merge}, we merge the partial results  while establishing the correct correspondence of the columns. Theoretical intuition as to why this is possible follows as a proof sketch:

Following the example of $r=2$ of the previous paragraph, according to Algorithm \ref{alg:merge}, we compute the inner product of the common parts of each column of $\mathbf{A}^{(1)}$ and $\mathbf{A}^{(2)}$. Since the common parts of each column are normalized to unit norm, then the inner product of the common part of the column of $\mathbf{A}^{(1)}$ with that of $\mathbf{A}^{(2)}$ will be maximized (and exactly equal to 1) for the matching columns, and by the Cauchy-Schwartz inequality, for all other combinations, it will be less than 1.

\begin{center}
\begin{algorithm} [ht]
\caption{\method: Fast, sparse, and parallel CMTF} \label{alg:repetition}
\begin{algorithmic}[1]
\REQUIRE Tensor $\tensor{X}$ of size $I\times J \times K$, matrices $\mathbf{Y}_i,~i=1\cdots 3$, of size $I\times I_2$, $J\times J_2$, and $K\times K_2$ respectively, number of factors $F$, sampling factor $s$, number of repetitions $r$.
\ENSURE $\mathbf{A}$ of size $I \times F$, $\mathbf{b}$ of size $J\times F$, $\mathbf{c}$ of size $K\times F$,
 $ \mathbf{D} $ of size $I_2\times F$, $\mathbf{G}$ of size $J_2 \times F$, $\mathbf{E}$ of size $K_2\times F$.
 $\blambda_A$, $\blambda_B$, $\blambda_C$, $\blambda_D$, $\blambda_E$, $\blambda_G$ of size $F\times 1$ which contains the scale of each component for each factor matrix.
\STATE Initialize $\mathbf{A,B,C,D,E,G}$ to all-zeros.
\STATE Randomly, \emph{using mode densities as bias}, select a set of $100p\%$ ($p \in [0,1])$ indices $\mathcal{I}_p,\mathcal{J}_p,\mathcal{K}_p$ to be common across all repetitions. For example, $\mathcal{I}_p$ is sampled with probabilities with  $p_\mathcal{I}(i) = \mathbf{x}_A (i) / \displaystyle{\sum_{i=1}^I \mathbf{x}_A (i)}$. Probabilities for the rest of the modes are calculated similarly.
\FOR {$i = 1 \cdots r$}

 \LCOMMENT {{\bf Phase 1: Sample indices}}	
\STATE Compute densities as in equations \ref{eq:dens_tensor}, \ref{eq:dens_mat_I}, \ref{eq:dens_mat_J}.

  Compute set of indices $\mathcal{I}^{(i)}$ as random sample without replacement of $\{1\cdots I\}$ of size $I/\left( s \left(  1-p\right)\right)$ with probability $p_\mathcal{I}(i) = \mathbf{x}_A (i) / \displaystyle{\sum_{i=1}^I \mathbf{x}_a (i)}$.
 Likewise for $\mathcal{J,K},$, $\mathcal{I}_1$, $\mathcal{I}_2$, and $\mathcal{I}_3$.
 Set $\mathcal{I}^{(i)} = \mathcal{I} \cup \mathcal{I}_p$. Likewise for the rest.


  \STATE Get $\tensor{X}^{(i)}_s = \tensor{X}(\mathcal{I}^{(i)},\mathcal{J}^{(i)},\mathcal{K}^{(i)})$, $\mathbf{Y}_{1s}^{(i)} = \mathbf{Y}_1(\mathcal{I}^{(i)},\mathcal{I}_1^{(i)} )$ and likewise for $\mathbf{Y}^{(i)}_{2s}$ and $\mathbf{Y}^{(i)}_{3s}$. Note that the same index sample is used for \emph{coupled} modes.

  \LCOMMENT {{\bf Phase: Fit the model on the sampled data}}
   \STATE  Run Algorithm \ref{alg:als} for $\tensor{X}_s^{(i)}$ and $\mathbf{Y}_{js}^{(i)},~j=1\cdots 3$ and obtain $\mathbf{A}_s,\mathbf{B}_s,\mathbf{C}_s, \mathbf{D}_s, \mathbf{G}_s, \mathbf{E}_s$.
	
	\STATE $\mathbf{A}^{(i)}(\mathcal{I}^{(i)},:) = \mathbf{A}_s$. Likewise for the rest.

	\STATE Calculate the $\ell_2$ norm of the columns of the common part: $\blambda^{(i)}_A(f)= \|\mathbf{A}^{(i)}(\mathcal{I}_p,f)\|_2 $, for $f=1\cdots F$.
	Normalize columns of $\mathbf{A}^{(i)}$ using $\blambda^{(i)}_A$ (likewise for the rest).
	Note that the common part of each factor will now be normalized to unit norm. 
\ENDFOR
\LCOMMENT {{\bf Phase 3: Merge partial results}}
\STATE $\mathbf{A} =$\merge$(\mathbf{A}_1^i)$. Likewise for the rest.
\STATE $\blambda_A = $ average of ${\blambda_A}_1^i$. Likewise for the rest.

\end{algorithmic}
\end{algorithm}
\end{center}

\begin{center}
\begin{algorithm} [ht]
\caption{\merge: Given partial results of factor matrices, merge them correctly} \label{alg:merge}
\begin{algorithmic}[1]
\REQUIRE Factor matrices $\mathbf{A}_1^i$ of size $I\times F$ each, and $r$ is the number of repetitions, $\mathcal{I}_p$: set of common indices.
\ENSURE Factor matrix $\mathbf{A}$ of size $I\times F$.
\STATE Set $\mathbf{A} = \mathbf{A}^{(1)}$
\STATE Set $\ell = \{1\cdots F\}$, a list that keeps track of which columns have not been assigned yet.
\FOR{$i = 2\cdots r$}
		\FOR {$f_1 = 1\cdots F$}
		\FOR{$f_2$ in $\ell$}
			\STATE Compute similarity $\mathbf{v}(f_2) = \left(\mathbf{A}(\mathcal{I}_p,f_2)\right) ^T \left(\mathbf{A}^{(i)}(\mathcal{I}_p,f_1))\right)$
		\ENDFOR
			\STATE $ c^*= \arg \max_{c} \mathbf{v}(c)$ (Ideally, for the matching columns, the inner product should be close to 1; conversely, for the rest of the columns, it should be considerably smaller)
			\STATE $\left. \mathbf{A}(:,c^*) =  \mathbf{A}^{(i)}(:,f_1)\right|_{\mathbf{A}(:,c^*)=0}$, i.e. update the zero entries of the column.
			\STATE Remove $c^*$ from list $\ell$.
		\ENDFOR
\ENDFOR
\end{algorithmic}
\end{algorithm}
\end{center}

\subsection{Speeding up the core of the algorithm}
In addition to our main contribution in terms of speeding up the decomposition, i.e. Algorithm \ref{alg:repetition}, we are able
to further speed the algorithm up, by making a few careful interventions to the core algorithm (Algorithm \ref{alg:als}).

\begin{lemma}
We may do the following simplification to each pseudoinversion step of the ALS algorithm (Algorithm \ref{alg:als}):
\[
	\begin{bmatrix}\mathbf{A\odot B} \\ \mathbf{M}\end{bmatrix}^\dag = \left( \mathbf{A}^T \mathbf{A} \ast \mathbf{B}^T \mathbf{B} + \mathbf{M}^T \ast \mathbf{M} \right)^\dag \begin{bmatrix} \left( \mathbf{A \odot B} \right)^T, \mathbf{M}^T \end{bmatrix}
\]
\end{lemma}
\begin{proof}
For the Moore-Penrose pseudoinverse of the Khatri-Rao product, it holds that \cite{bro1998multi}, \cite{liu2008hadamard}
\[
	\left( \mathbf{A \odot B}  \right)^\dag = \left( \mathbf{A}^T \mathbf{A} \ast \mathbf{B}^T \mathbf{B} \right)^\dag \left( \mathbf{A \odot B} \right)^T
\]
Furthermore \cite{bro1998multi}
\[
	\left(  \mathbf{A \odot B} \right)^T \left(  \mathbf{A \odot B} \right) = \mathbf{A}^T \mathbf{A} \ast  \mathbf{B}^T \mathbf{B}
\]	
For a partitioned matrix $\mathbf{P} = \begin{bmatrix} \mathbf{P}_1 \\ \mathbf{P}_2\end{bmatrix}$, it holds that its pseudoinverse may be written in the following form \cite{hung1975moore}
\[
	\begin{bmatrix} \mathbf{P}_1 \\ \mathbf{P}_2\end{bmatrix}^\dag = \left(  \mathbf{P}_1 ^T \mathbf{P}_1 + \mathbf{P}_2 ^T \mathbf{P}_2  \right)^\dag \begin{bmatrix} \mathbf{P}_1 ^T, & \mathbf{P}_2 ^T\end{bmatrix}
\]
Putting things together, it follows:
\[
	\begin{bmatrix}\mathbf{A\odot B} \\ \mathbf{M}\end{bmatrix}^\dag = \left( \mathbf{A}^T \mathbf{A} \ast \mathbf{B}^T \mathbf{B} + \mathbf{M}^T \ast \mathbf{M} \right)^\dag \begin{bmatrix} \left( \mathbf{A \odot B} \right)^T, \mathbf{M}^T \end{bmatrix}
\]
which concludes the proof.
\end{proof}

The above lemma implies that  substituting the naive pseudoinversion of $\begin{bmatrix}\mathbf{A\odot B} \\ \mathbf{M}\end{bmatrix}$ with the simplified version, offers significant \emph{computational} gains to Algorithm \ref{alg:als} and hence to Algorithm \ref{alg:repetition}. More precisely, if the dimensions of $\mathbf{A,B}$ and $\mathbf{M}$ are $I \times R$, $J \times R$ and $I \times I_2$, then computing the pseudoinverse naively would cost $O\left( R^2 \left(  IJ + I_2\right) \right)$, whereas our proposed method yields a cost of $O\left( R^2 \left( I + J + I_2\right)\right)$ because of the fact that we are pseudoinverting only a \emph{small} $R \times R$ matrix. We have to note here that in almost all practical scenarios $R\ll I, J, I_2$.

\subsection{Accounting for missing values}

In many practical scenarios, we often have corrupted or missing data. For instance, when measuring brain activity, a few sensors might stop working, whereas the majority of the sensors produce useful signal. Despite these common data imperfections, it is important for a data mining algorithm to be able to operate. 



\hide{ 
Traditionally, there are two major ways of doing so. Either we impute the missing values \emph{prior} to the factorization, or we essentially \emph{ignore} the missing values during the factorization process. Attempting to estimate the missing values is a very hard {\em application-specific} task, since it requires domain knowledge of the underlying data, and is usually at least as hard (algorithmically and computationally) as carrying out a factorization. On the other hand, ignoring missing data elements in the factorization process is a cleaner and general way of handling missing values without making any assumptions on the underlying data model, and only requires a few careful modifications of the existing algorithm.
}

We carefully ignore the missing values 
from the entire optimization procedure:
Notice that is \emph{not} the same as simply zeroing out all missing values, 
since 0 might have a valid physical interpretation.
Specifically,
we define a 'weight' tensor $\tensor{W}$ which has '0' in all coefficients where values are missing, and '1' everywhere else. Similarly, we introduce three weight matrices $\mathbf{W}_i$ for each of the coupled matrices $\mathbf{Y}_i$. Then, the optimization function of the CMTF model becomes

\begin{align*}
	&\displaystyle{ \min_{\mathbf{A,B,C,D,E,G}} \| \tensor{W}\ast \left( \tensor{X} - \sum_k \mathbf{a}_k \circ \mathbf{b}_k\circ \mathbf{c}_k \right)\|_F^2  ~+  }  \\
	&\displaystyle{\|  \mathbf{W}_1 \ast \left( \mathbf{Y}_1 - \mathbf{AD}^T \right)\|_F^2   + \| \mathbf{W}_2 \ast\left( \mathbf{Y}_2 - \mathbf{BE}^T \right)\|_F^2  ~+}\\
	&   \displaystyle{  \| \mathbf{W}_3 \ast \left( \mathbf{Y}_3 - \mathbf{CG}^T \right) \|_F^2}
\end{align*}

As we show in Algorithm \ref{alg:als}, we may solve CMTF by solving six least squares problems in an alternating fashion. A fortuitous implication of this fact is that in order to handle missing values for CMTF, it suffices to solve

\begin{equation}
	\min_{\mathbf{B}} \|  \mathbf{W}\ast \left( \mathbf{X} -  \mathbf{AB}^T  \right)\|_F^2
	\label{eq:lsmv}
\end{equation}
where $\mathbf{W}$ is a weight matrix in the same sense as described a few lines earlier.

On our way tackling the above problem, we first need to investigate its scalar case, i.e. the case where we are interested only in $\mathbf{B}(j,f)$ for a fixed pair of $j$ and $f$. The optimization problem may be rewritten as
\[
	\min_{\mathbf{B}(j,f)} \|\mathbf{W}(:,j)\ast\mathbf{X}(:,j) - \left(\mathbf{W}(:,j)\ast \mathbf{A}(:f)\right) \mathbf{B}(j,f)^T  \|
\]
which is essentially a scalar least squares problem of the form:
\[
	\min_{b} \| \mathbf{x} - \mathbf{a}b\|_2^2
\]
with solution in analytical form: $b = \frac{\mathbf{x}^T \mathbf{a}}{ \| \mathbf{a} \|_2^2   }$

We may, thus, solve this problem of Equation \ref{eq:lsmv} using \emph{element-wise coordinate descent}, where we update each coefficient of $\mathbf{B}$ iteratively, until convergence. Therefore, with the aforementioned derivation, we are able to modify our original algorithm in order to take missing values into account.

\subsection{Parallelization}
Our proposed algorithm is, by its nature, parallelizable; in essence, we generate multiple samples of the coupled data, we fit a CMTF model to each sample and then we merge the results. By carefully observing Algorithm \ref{alg:repetition}, we can see that lines 3 to 9 may be carried out entirely in parallel, provided that we have a good enough random number generator that does not generate the very same sample across all $r$ repetitions. In particular, the $r$ repetitions are independent from one another, since computing the set of common indices (line 2), which is the common factor across all repetitions, is done before line 3.

\section{Knowledge Discovery}
\label{sec:discovery}
\subsection{\methodplain on Brain Image Data With Additional Semantic Information}

As part of a larger study of neural representations of word meanings in the human brain \cite{mitchell2008predicting}, we applied \methodplain to a combination of datasets which we henceforth jointly refer to as \brain.  This dataset consists of two parts.  The first is a tensor that contains measurements of the fMRI brain activity of 9 human subjects, when shown each of 60 concrete nouns (5 in each of 12 categories, e.g. dog, hand, house, door, shirt, dresser, butterfly, knife, telephone, saw, lettuce, train).   fMRI measures slow changes in blood oxygenation levels, reflecting localized changes in brain activity.  Here our data is made up of $3\times3\times6$mm voxels (3D pixels) corresponding to fixed spatial locations across participants.  Recorded fMRI values are the mean activity over 4 contiguous seconds, averaged over multiple presentations of each stimulus word (each word is presented 6 times as a stimulus). Further acquisition and preprocessing details are given in \cite{mitchell2008predicting}. This dataset is publicly available\footnote{\url{http://www.cs.cmu.edu/afs/cs/project/theo-73/www/science2008/data.html}}.  The second part of the data is a matrix containing answers to 218 questions pertaining to the semantics of these 60 nouns. A sample of these questions is shown in Table \ref{tab:fmri}.

This dataset has been used before in works such as \cite{murphy2012selecting}, \cite{palatucci2009zero}.

\brain's size is 60$\times77775\times$9  with over 11 million non-zeros (tensor), and $60\times 218$ with about 11.000 non-zeros (matrix). The dimensions might not be extremely high, however, the data is \emph{very dense} and it is therefore difficult to handle efficiently. For instance, decomposing the dataset using the simple ALS algorithm took more than 24 hours, whereas \method yielded a speedup of 50-100$\times$ over this (cf. Figure \ref{fig:crown}).

\noindent{\bf Simultaneous Clustering of Words, Questions and Regions of the Brain}

One of the strengths of our proposed method is its expressiveness in terms of simultaneously soft-clustering
all involved entities of the problem. By taking a low rank decomposition of the \brain data (using $r = 5$ and $s_I = 3,~ s_J= 86, ~s_K=   1$ for the tensor and $s_I$ for the questions dimension of the matrix)\footnote{We may use imbalanced sampling factors, especially when the data is far from being 'rectangular'.}, we are able to find groups that jointly express words, questions and brain voxels (we can also derive groups of human subjects; however, it is an active research subject in neuroscience, whether brain-scans should differ significantly between people, and is out of the scope of the present work).

 In Figure \ref{fig:fmri_results}, we display 4 such groups of brain regions that are activated given a stimulus of a group of similar words; these words can be seen in Table \ref{tab:fmri}, along with groups of similar questions that were highly correlated with the words of each group. Moreover, we were able to successfully identify high activation of the \emph{premotor cortex} in Group 3, which is associated with concepts such as holding or picking items up.

\begin{figure*}[htbf]
	\begin{center}
		\subfigure[Group 1]{\includegraphics[width = 0.245\textwidth]{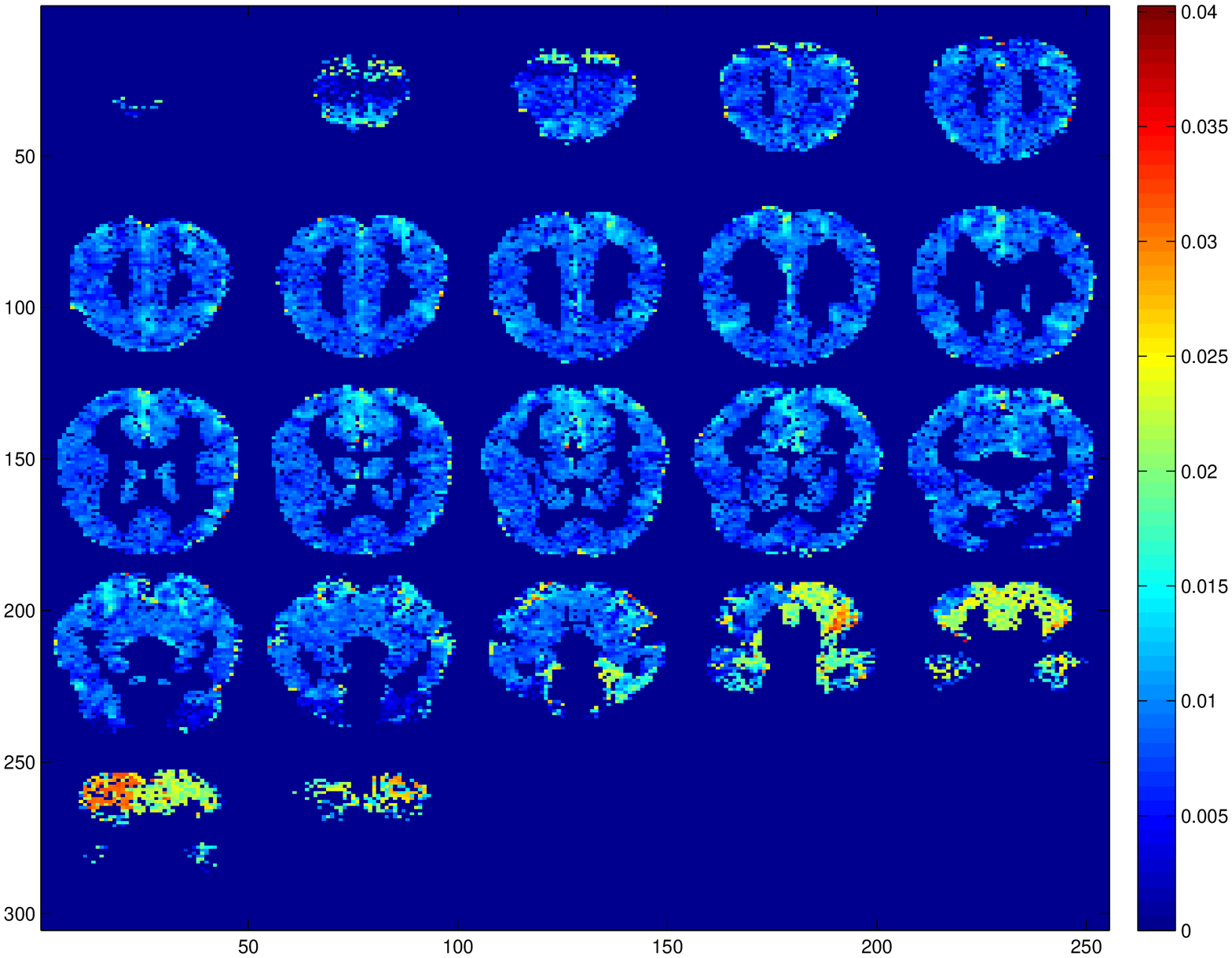}}
		\subfigure[Group 2]{\includegraphics[width = 0.245\textwidth]{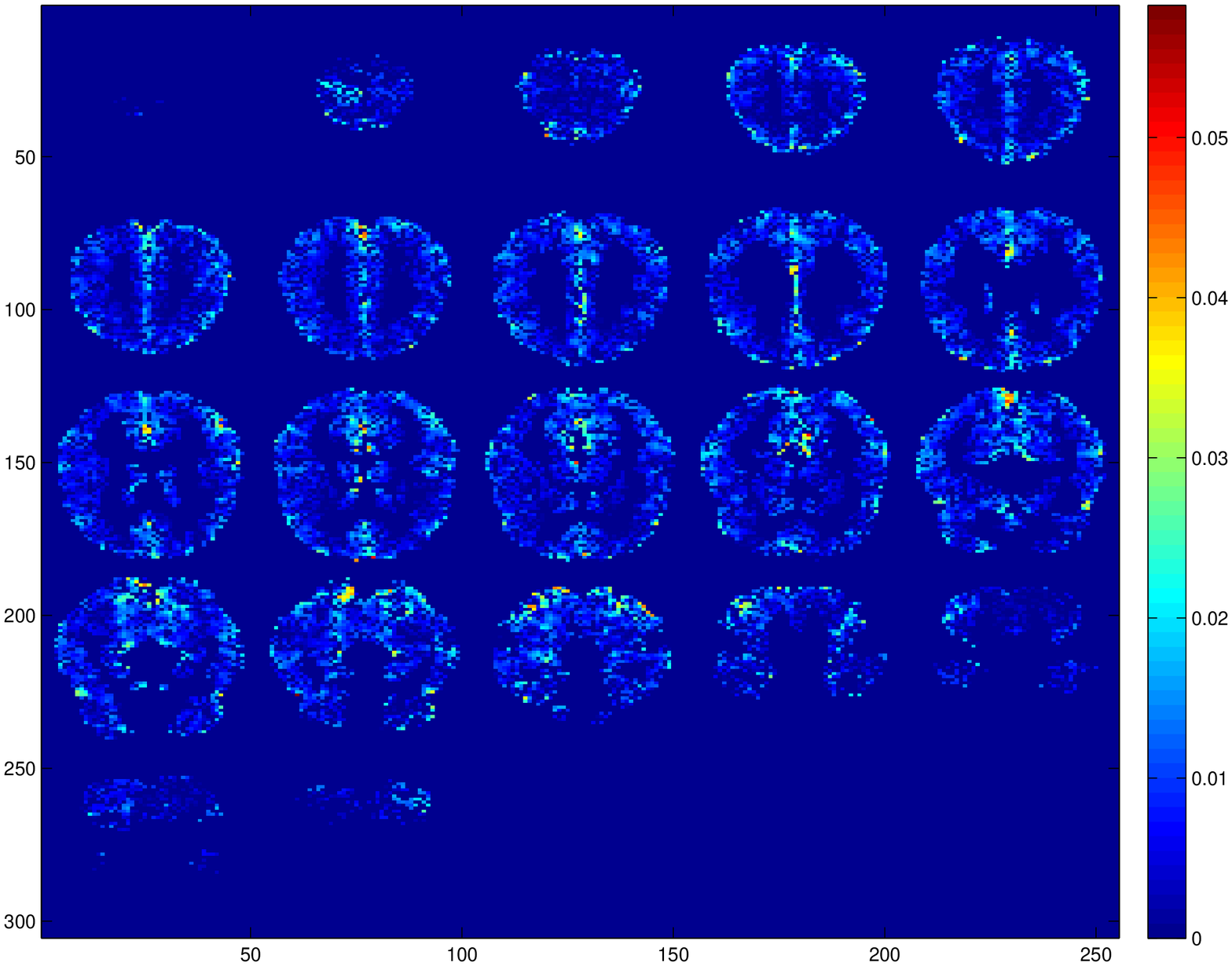}}
		\subfigure[Group 3]{\includegraphics[width = 0.245\textwidth]{FIG/fmri/6.eps}}
		\subfigure[Group 4]{\includegraphics[width = 0.245\textwidth]{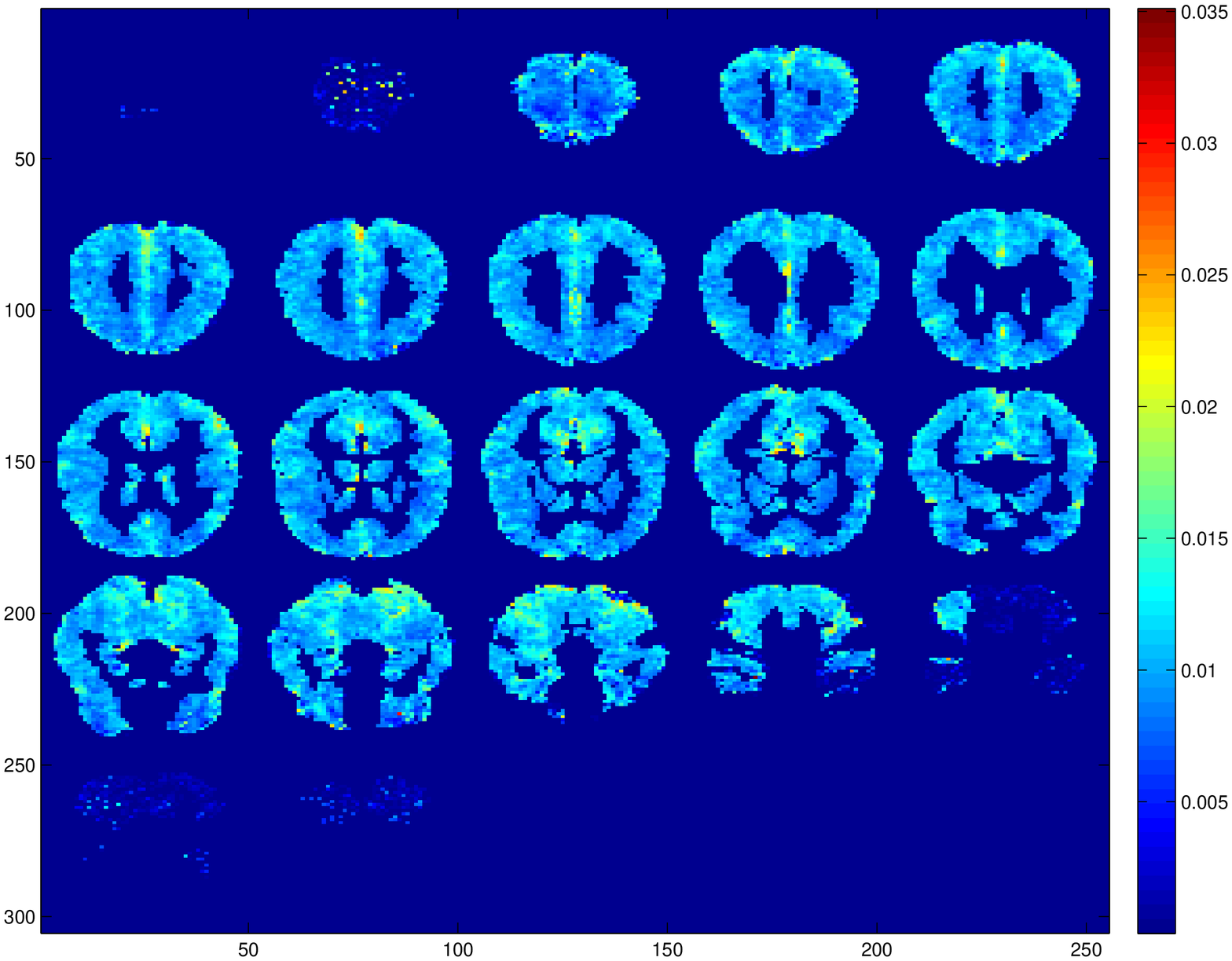}}
	\end{center}
	\caption{The latent brain images for the 4 word/question groups as shown in Table \ref{tab:fmri}. We can see that for each different group, the activation pattern of certain brain regions is different. For instance, Group 3 refers to small items that can be held in one hand,such as a tomato or a glass, and the activation pattern is very different from the one of Group 1, which mostly refers to insects, such as bee or beetle. Additionally, Group 3, for instance, shows high activation in the \emph{premotor cortex} which is associated with the concepts of that group.}
	\label{fig:fmri_results}
\end{figure*}

 \begin{table}[htpf]
  \begin{center}
{
  \begin{tabular}{lll}
     &\textbf{Nouns} & \textbf{Questions} \\ \hline
	\multirow{4}{*}{Group1} & beetle & can it cause you pain? \\
 & pants & do you see it daily? \\
 & bee & is it conscious? \\ \hline
	\multirow{4}{*}{Group 2} &bear & does it grow?\\
		&cow	& is it alive?	\\
		&coat	&	was it ever alive?\\ \hline
	\multirow{4}{*}{Group 3} &glass	& can you pick it up?	\\
	&tomato	& 	can you hold it in one hand?\\
	&bell	&	is it smaller than a golfball?'\\ \hline
	\multirow{4}{*}{Group 4}&	bed & does it use electricity?	\\
	&	house &	can you sit on it?\\
	&	car &	does it cast a shadow? \\ \hline
  \end{tabular}
  }
  \end{center}
  \caption{Groups of nouns and questions that are both positively and negatively correlated, and correspond to the brain regions show in Fig. \ref{fig:fmri_results}.}
  \label{tab:fmri}
\end{table}

\noindent{\bf By-product: Predicting Brain Activity from Questions}

In addition to soft-clustering, the low rank joint decomposition of the \brain data offers another significant result. This low dimensional embedding of the data into a common semantic space, enables the prediction of, say, the brain activity of a subject, for a given word, given the corresponding vector of question answers for that word. In particular, by projecting the question answer vector to the latent semantic space and then expanding it to the brain voxel space, we obtain a fairly good prediction of the brain activity.

To evaluate the accuracy of these predictions of brain activity, we follow a \emph{leave-two-out} scheme, where we remove two words entirely from the brain tensor and the question matrix; we carry out the joint decomposition, in some very low dimension, for the remaining set of words and we obtain the usual set of matrices $\mathbf{A,B,C,D}$. Due to the randomized nature of \method, we did 100 repetitions of the procedure described below.

Let $\mathbf{q}_{i}$ be the question vector for some word $i$, and $\mathbf{v}_i$ be the brain activity of one human subject, pertaining to the same word. By left-multiplying $\mathbf{q}_i$ with $\mathbf{D}^T$, we project $\mathbf{q}_i$ to the latent space of the decomposition; then, by left-multiplying the result with $\mathbf{B}$, we project the result to the brain voxel space. Thus, our estimated (predicted) brain activity is obtained as
$\mathbf{\hat{v}}_i = \mathbf{B D}^T \mathbf{q}_i$

Given the predicted brain activities $\mathbf{\hat{v}_1}$ and $\mathbf{\hat{v}_2}$ for the two left out words, and the two actual brain images $\mathbf{v_1}$ and $\mathbf{v_2}$ which were withheld from the training data, the \emph{leave-two-out} scheme measures prediction accuracy by the ability to choose which of the observed brain images corresponds to which of the two words.  After mean-centering the vectors, this classification decision is made according to the following rule:
\[
	\| \mathbf{v}_1 - \mathbf{\hat{v}}_1 \|_2 + \| \mathbf{v}_2 - \mathbf{\hat{v}}_2 \|_2 < \| \mathbf{v}_1 - \mathbf{\hat{v}}_2 \|_2 + \| \mathbf{v}_2 - \mathbf{\hat{v}}_1 \|_2
\]

Although our approach is
not designed to make predictions, preliminary results are very encouraging: Using only $F$=2 components, for the noun pair \emph{closet/watch} we obtained mean accuracy of about 0.82 for 5 out of the 9 human subjects. Similarly, for the pair \emph{knife/beetle}, we achieved accuracy of about 0.8 for a somewhat different group of 5 subjects. For the rest of the human subjects, the accuracy is considerably lower, however, it may be the case that brain activity predictability varies between subjects, a fact that requires further investigation.

We plan detailed experiments to determine the accuracy of these predictions
compared to specialized methods that have previously been used for
these predictions, but which do not have the ability of our method to
discover latent representations, such as \cite{palatucci2009zero}.


\hide{
\begin{figure}[htbf]
	\begin{center}
		\includegraphics[width = 0.45\textwidth]{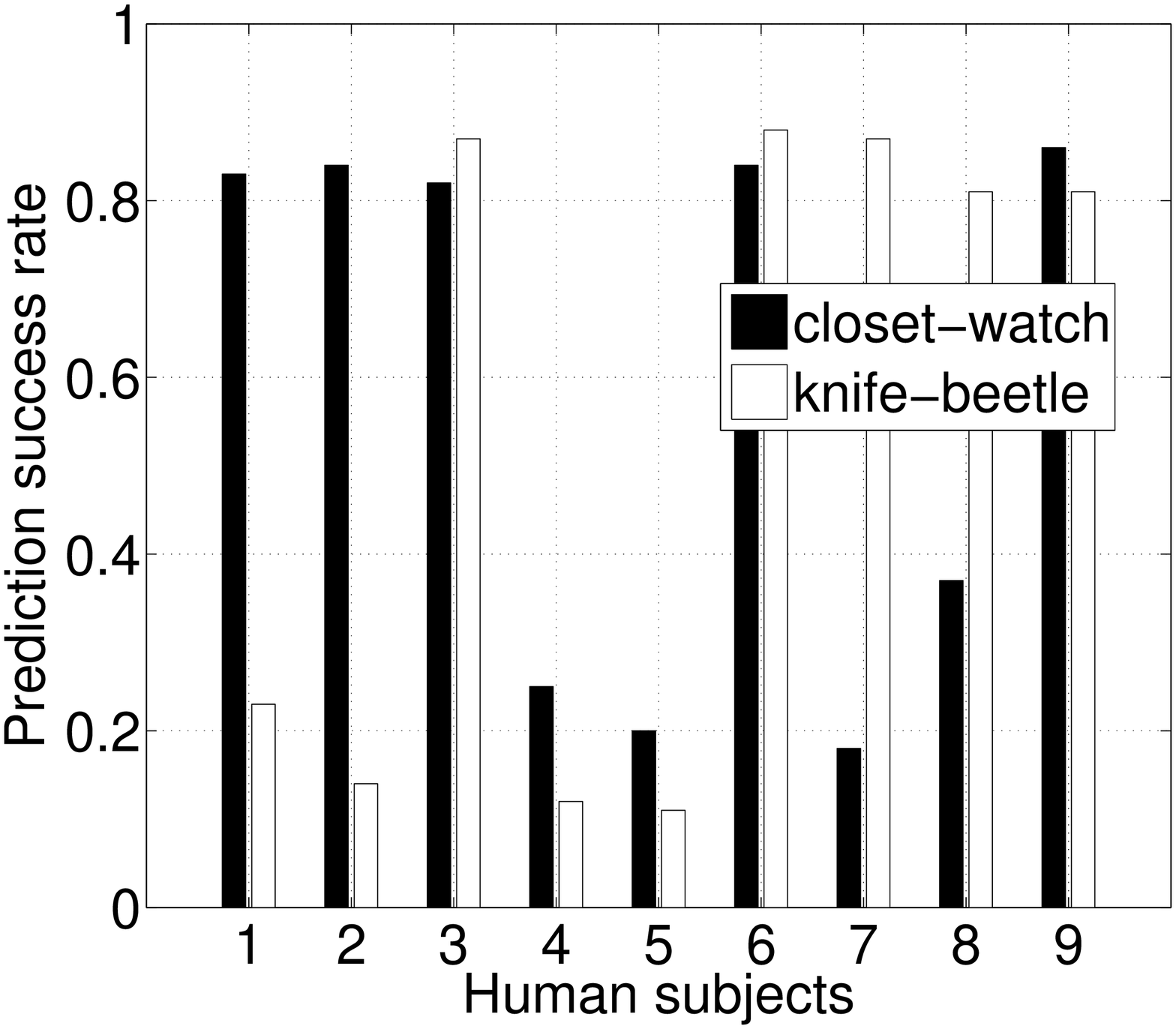}
		\caption{Prediction accuracy for two pairs of words, using \method and a leave-two-words-our cross-validation framework. The accuracy is computed per human subject basis; these results are on par with state of the art techniques, albeit our method was not specifically designed for this task, but merely offers it as a by-product.}
		\label{fig:classification}
	\end{center}
\end{figure}
}


\subsection{Generality: Mining Social Networks with Additional Information}
We have demonstrated the expressive power of \method for the \brain dataset, but in this subsection, we stress the fact that the method is actually application independent and may be used in vastly different scenarios. To that end, we analyze a \facebook dataset, introduced in \cite{viswanath-2009-activity}\footnote{Download \facebook at \url{http://socialnetworks.mpi-sws.org/data-wosn2009.html}}. This dataset consists of a $63890\times 63890\times 1847$ (wall, poster, day) tensor with about 740.000 non-zeros, and a $63890\times 63890$ who is friends with whom matrix, with about 1.6 million non-zeros. In contrast to \brain, this dataset is very sparse (as one would expect from a social network dataset). However, \method works in both cases, demonstrating that it can analyze data efficiently, regardless of their density.

We decomposed the data into 25 rank one components, using  $s_I=1000,~ s_J=1000, s_K =100$ and $s_I$ for both dimensions of the matrix, and manually inspected the results. A fair amount of components captured normal activity of Facebook users who occasionally post on their friends' walls; here we only show one outstanding anomaly, due to lack of space: In Fig. \ref{fig:fb_spammer} we show what appears to be a spammer, i.e. a person who, only on a certain day, posts on many different friends' walls: the first subfigure corresponds to the wall owners, the second subfigure corresponds to the people who post on these walls, and the third subfigure is the time (measured in days); we thus have one person, posting on many peoples' walls, on a single day.

\begin{figure}[htbf]
	\begin{center}
		\includegraphics[width = 0.45\textwidth]{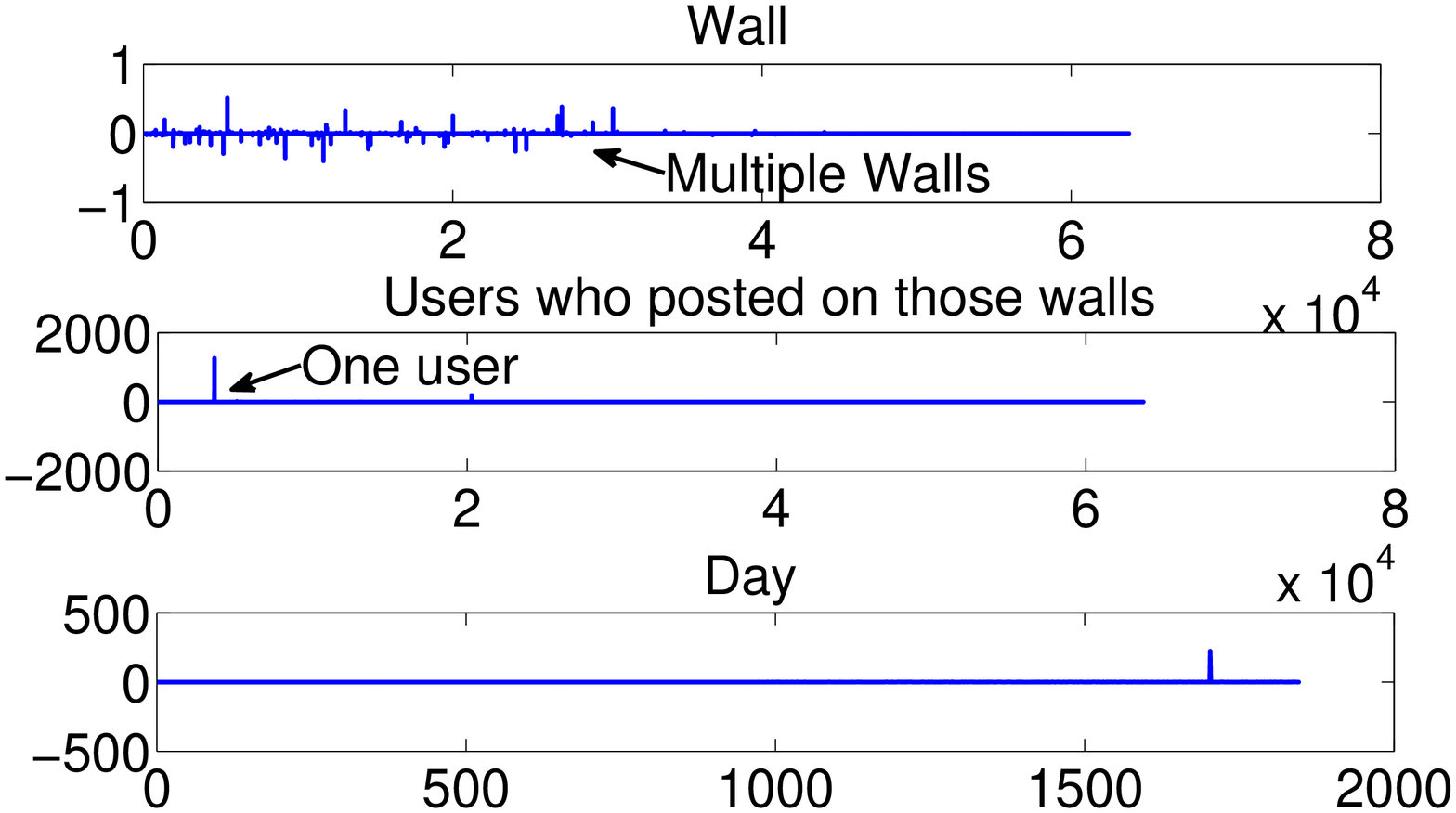}
		\vspace{-0.28in}
		\caption{This is a pattern extracted using \method, which shows what appears to be a spammer on the \facebook dataset: One person, posting to many different walls on a single day. }
		\label{fig:fb_spammer}
	\end{center}
\end{figure}

\section{Experiments}
\label{sec:exp}
We implemented \method in Matlab. For the parallelization of the algorithm, we used Matlab's Parallel Computing Toolbox. For tensor manipulation, we used the Tensor Toolbox for Matlab \cite{tensortoolbox} which is optimized especially for sparse tensors (but works very well for dense ones too). All experiments were carried out on a machine with 2 dual-core AMD Opteron 880 processors (2.4 GHz), 4 TB disk, and 48GB ram.  The parallel experiments were run on all 4 cores, which justifies our choice of  $r=4$ in this case. Whenever we conducted multiple iterations of an experiment (due to the randomized nature of \method), we report error-bars along the plots. For all the following experiments we used either portions of the \brain dataset, or the whole dataset.

\subsection{Accuracy}
In Figure \ref{fig:errors} we demonstrate that the algorithm operates correctly, in the sense that it reduces the model cost (Equation \ref{eq:main_obj}) when doing more repetitions. In particular, the vertical axis displays the relative cost, i.e.
$
	\frac{\text{\method cost}}{\text{ALS cost}}
$
(with ideal being equal to 1) and the horizontal axis is the number of repetitions in the sampling.
We observed that for a few executions of the algorithm, the cost was not monotonically decreasing; however, we 
ran the algorithm 1000 times, keeping the executions that decreased the relative cost monotonically and plotted them in Fig. \ref{fig:errors}. 

\begin{figure}[htb]
	\begin{center}
	
		\includegraphics[width = 0.5\textwidth]{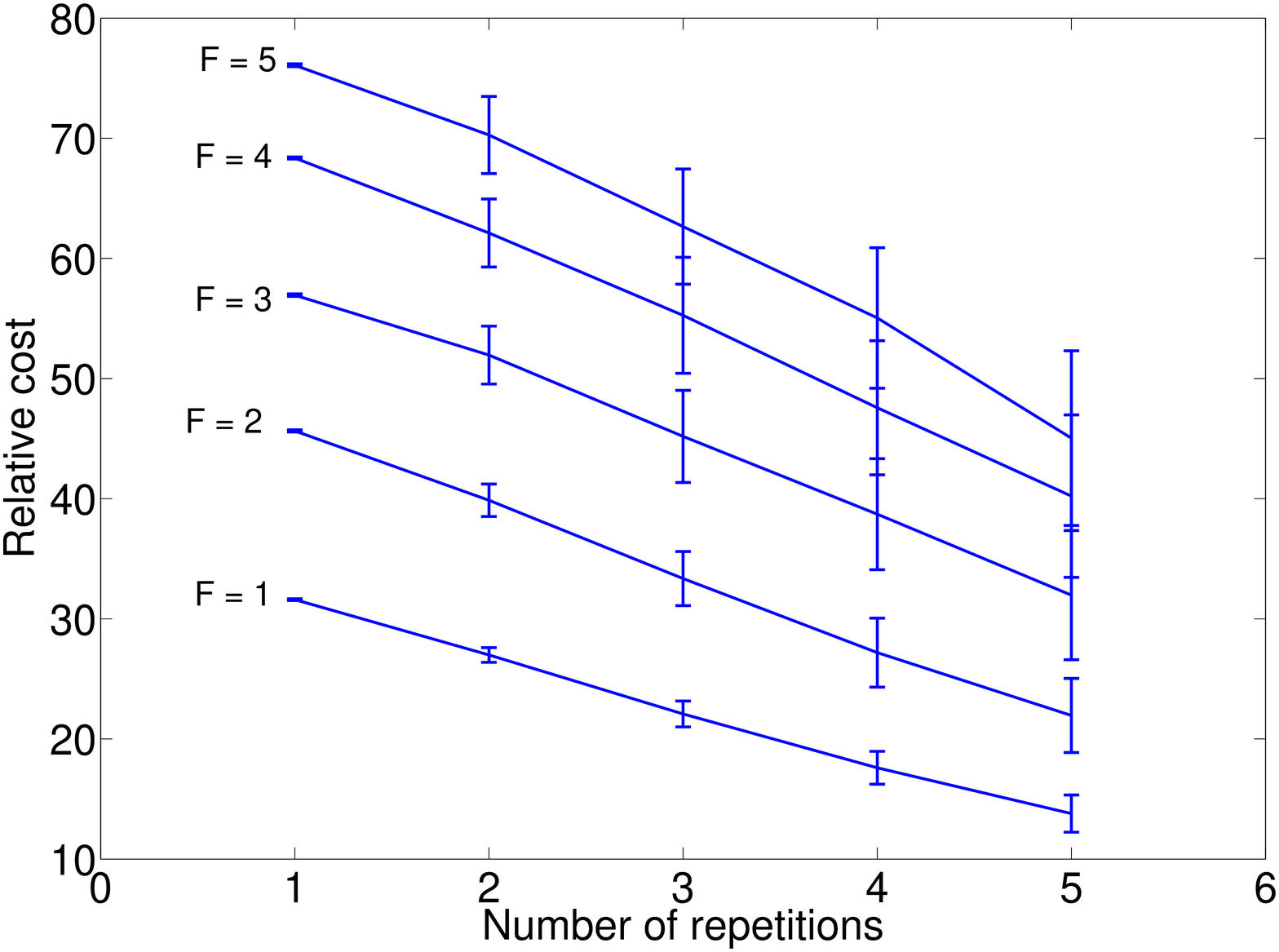}
		\caption{The relative cost of the model, as a function of the number of repetitions $r$ is decreasing, which empirically shows that \method actually reduces the approximation error of the CMTF model.}
		\label{fig:errors}
	\end{center}
\end{figure}

\subsection{Speedup}
As we have already discussed in the Introduction, \method achieves a speedup of 50-100 on the \brain; for the 50$\times$ case, the same approximation error of the CMTF objective is maintained, while for higher speedup values, the relative cost increases, but within reasonable range. Figure \ref{fig:crown} illustrates this behaviour.

Additionally, \method benefits greatly from it's inherent parallelizability. The parallel results we report come from $r=4$ repetitions of sampling, carried out on 4 cores; had more cores been available, we would probably observe a higher speedup (keeping of course Amdahl's law in mind), while maintaining low relative cost, since we establish in the previous subsection that the more repetitions we do, the better we approximate the CMTF model.

\subsection{Sparsity}
One of the main advantages of \method is that, \emph{by construction}, it produces \emph{sparse} latent factors for coupled matrix-tensor model. 
In Fig. \ref{fig:sparsity} we demonstrate the sparsity of \method 's results by introducing the relative sparsity metric; this intuitive metric is simply the ratio of the output size of the ALS algorithm, divided by the output size of \method. The output size is simply calculated by adding up the number of non-zero entries for all factor matrices output by the algorithm. We use a portion of the \brain dataset in order to execute this experiment. We can see that for the dense \brain dataset, we obtained twice as sparse results. However, in experiments with randomly generated, sparse, data, we experienced higher degrees of sparsity, in the order of 5$\times$. We omit such plots due to space constraints.


\begin{figure}[htbf]
	\begin{center}
		\includegraphics[width = 0.5\textwidth]{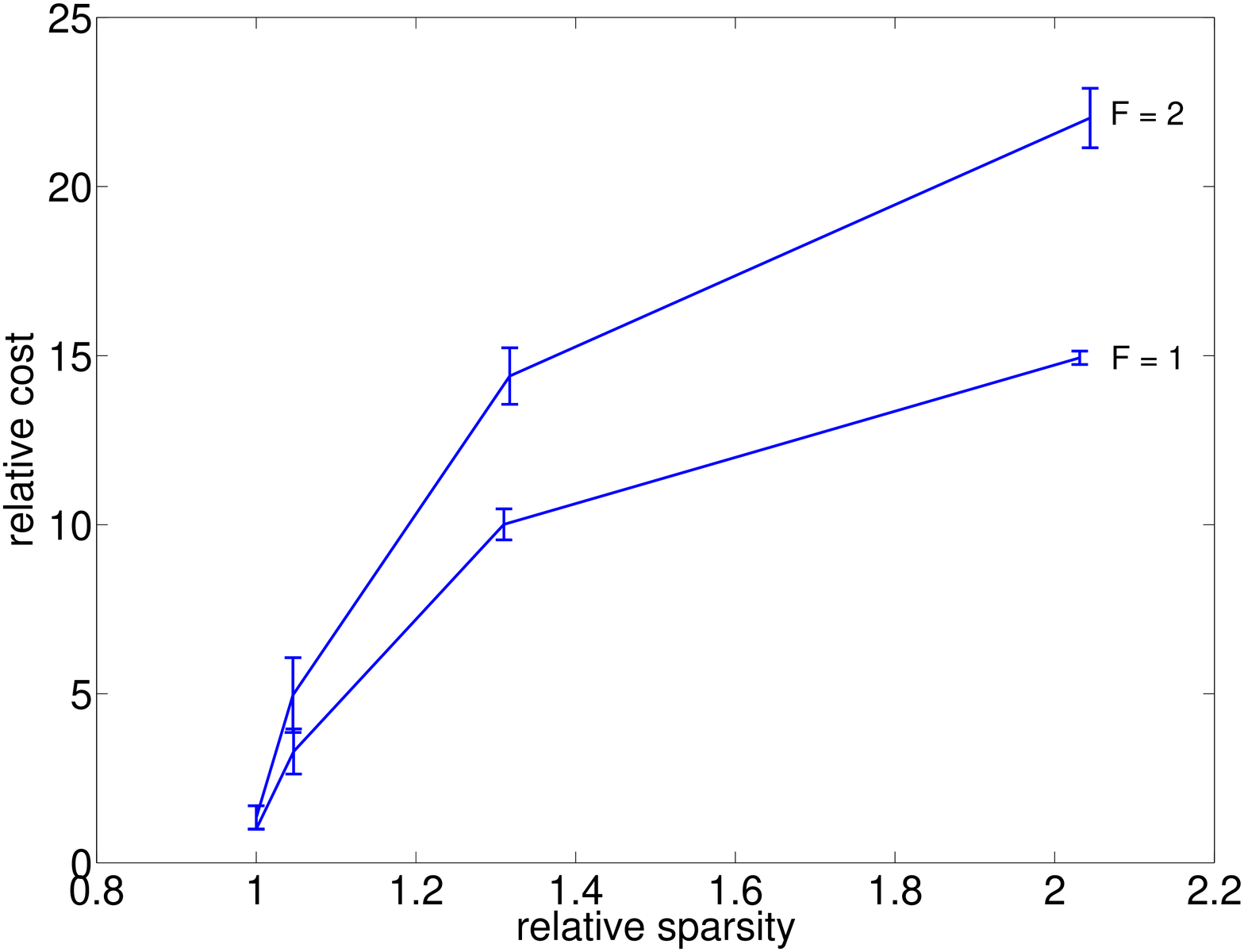}
			\end{center}
		\caption{The relative output size vs. the relative cost indicates that, even for very dense datasets such as \brain, we are able to get a 2 fold decrease in the output size, while maintaining good approximation cost.}
		\label{fig:sparsity}
\end{figure}

\subsection{Robustness to missing values}
In order to measure resilience to missing values we define the \emph{Signal-to-Noise Ratio} (SNR) as simply as $\text{SNR} = \frac{\| \tensor{X}_m\|_F^2}{ \| \tensor{X}_m - \tensor{X}_0  \|_F^2}$, where $\tensor{X}_m$ is the reconstructed tensor when a $m$ fraction of the values are missing.
In Figure \ref{fig:mv}, we demonstrate the results of that experiment; we observe that even for a fair amount of missing data, the algorithm performs reasonably well, achieving high SNR. Moreover, for small amounts of missing data, the speed of the algorithm is not degraded, while for larger values, it is considerably slower, probably due to Matlab's implementation issues. However, this is encouraging, in the sense that if the amount of missing data is not overwhelming, \method is able to deliver a very good approximation of the latent subspace. This experiment was, again, conducted on a portion of \brain.

\begin{figure}[htp]
	\begin{center}
		\includegraphics[width = 0.5\textwidth]{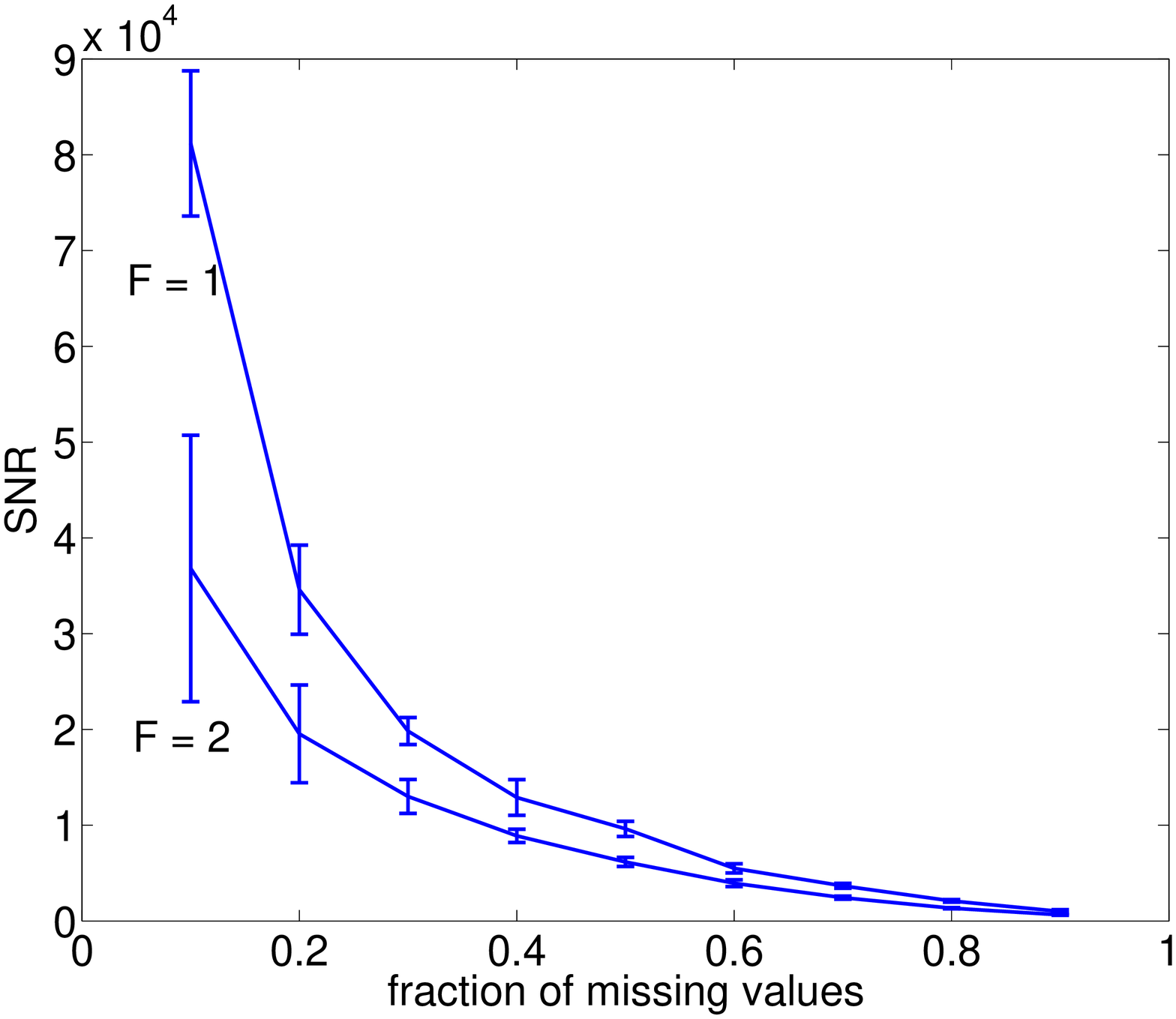}
		\caption{This Figure shows the Signal-to-Noise ratio (SNR)-as defined in the main text- as a function of the percentage of missing values. We can observe that, even for a fair amount of missing values, the SNR is quite high, signifying that \method is able to handle such ill-conditioned settings, with reasonable fidelity.}
		\label{fig:mv}
			\end{center}

\end{figure}

\section{Related Work}
\label{sec:related}
\noindent{\bf Coupled, Multi-block, Multi-set Models}
Coupled Matrix-Tensor Factorizations belong to a family of models also referred to as \emph{Multi-block} or \emph{Multi-set} in the literature.
Smilde et al. in \cite{smilde2000multiway} provided the first disciplined treatment of such multi-block models, in a chemometric context. An important issue with these models is how to weigh the different data blocks such that scaling differences may be alleviated. In \cite{wilderjans2009simultaneous}, Wilderjans et al. propose and compare two different weighing schemes. Most related to the present work is the work of Acar et al. in \cite{acar2011all}, where a first order optimization approach is proposed, in order to solve the CMTF problem. As we mention in the Introduction, \method is compatible with this algorithm, since it provides an alternative to the core CMTF solver. In \cite{acar2012coupled}, Acar et. al apply the CMTF model, using the aforementioned first-order approach in a chemometrics setting.
In \cite{acar2012coupledMatrix}, Acar et. al introduce a coupled matrix decomposition, where two matrices match on one of the two dimensions, and are decomposed in the same spirit as in CMTF, while imposing explicit sparsity constraints (via $\ell_1$ norm penalties); although \method also produces sparse factors, this so happens as a fortuitous byproduct of sampling, whereas in \cite{acar2012coupledMatrix} an explicit sparsity penalty is considered.
As an interesting application, in \cite{zheng2010collaborative}, the authors employ CMTF for Collaborative Filtering.
On a related note,  \cite{yokota2012linked}, \cite{lin2009metafac}, and \cite{liumining} introduce models where multiple tensors are coupled with respect to one mode, and analyzed jointly; in this work, we don't consider coupling of two (or more) tensors, however, we leave that for future work.

Having listed an outline of relevant approaches, to the best of our knowledge, \method is the first algorithm for CMTF that combines speed, parallelization, as well as sparse factors. An alternative perspective on \method is that of a framework that is able to speed up and sparsify any (possibly highly fine tuned) core algorithm for CMTF.

\noindent{\bf Treating Missing Values in Tensor Decompositions}
Tomasi et. al \cite{tomasi2005parafac} provides a very comprehensive study on how to handle missing values for plain tensor decompositions.

\noindent{\bf Fast \& Scalable Tensor Decompositions}
In \cite{papalexakis2012parcube} we introduced a parallel algorithm for the regular PARAFAC decomposition, where a sampling scheme of similar nature as here is exploited;  in \cite{kang2012gigatensor}, a scalable MapReduce implementation of PARAFAC is presented. Additionally, the mechanics behind the Tensor Toolbox for Matlab \cite{tensortoolbox} are very powerful when it comes to memory-resident tensors. Finally, in \cite{zhang2009parallel}, the authors introduce a parallel framework in order to handle tensor decompositions efficiently.

\noindent{\bf Tensor applications to brain data}
There has been substantial related work, which utilizes tensors for this purpose, e.g.  \cite{cichocki2009nonnegative}, \cite{acar2007multiway}.

\section{Conclusions}
\label{sec:concl}
Our main contributions are the following:

\begin{itemize}[noitemsep]
	\item \emph{Fast, parallel  \& sparsity promoting algorithm:} 
	\method is up to {\em 50-100 times} faster than state of the art algorithms.
	
	\item \emph{Robustness to missing data}: \method can effectively handle missing values, without significant performance degradation, even for moderate amounts of missing entries. 
	\item \emph{Effectiveness and Knowledge Discovery}: 
	\method, applied to the \brain dataset, 
	discovers meaningful triple-mode clusters: 
	clusters of words, 
	of questions, and of brain regions have similar behavior; as a by-product, \method is able to predict brain activity with very promising preliminary results.
\item \emph{Generality}:
	We applied \method to a \facebook dataset with additional information, 
	identifying what appears to be a spammer.

\end{itemize}

\hide{
\begin{itemize}[noitemsep]
	\item \emph{Fast and scalable algorithm}: We develop a fast and scalable algorithm which may operate in conjunction with any given core implementation of the CMTF decomposition (we choose the ALS algorithm, however, there exist other approaches as well, as outlined in Section \ref{sec:related}). Additionally, our algorithm scales well regardless of the data sparsity, since it does not rely on, e.g. high data sparsity.
	\item \emph{Latent factor sparsity}: The decomposition factors are significantly sparser than the ones produced by the ALS algorithm, while maintaining a good approximation of the original tensor.
	\item \emph{Robustness to missing values}:
	\item \emph{Parallelizable}:
	\item \emph{Simple to implement}:
\end{itemize}

In addition to algorithm design, we provide a structured way to represent a complex dataset and make interesting observations and discoveries.
}

\section*{Acknowledgements}
Research was funded by grants NSF IIS-1247489,  NSF IIS-1247632, NSF CDI 0835797, and NIH/NICHD 12165321.  Any opinions, findings, and conclusions or recommendations expressed in this
   material are those of the author(s) and do not necessarily reflect the views
   of the National Science Foundation, or other funding parties. The authors would also like to thank Leila Wehbe and Alona Fyshe for their initial help with the \brain data.

\bibliography{BIB/vagelis-ref}
\bibliographystyle{plain}

\balancecolumns


\end{document}